\documentclass[letterpaper, 10pt, journal, twoside]{IEEEtran}  % Comment this line out if you need a4paper

\usepackage{graphicx} % for pdf, bitmapped graphics files
\usepackage{amsmath} % assumes amsmath package installed
\usepackage{amssymb}  % assumes amsmath package installed
\usepackage{gensymb}
\usepackage{bm}
\usepackage{multirow}
\usepackage{subcaption}
\usepackage{xcolor}
\usepackage{hyperref}

\usepackage{multicol}
\usepackage{multirow}
\usepackage{tabularx}
\usepackage{epstopdf}

\usepackage{textpos}

\usepackage{amsthm}

\usepackage{bbm}
\usepackage{dsfont}

\usepackage{algorithm}
\usepackage{algpseudocode}
\makeatletter
\let\OldStatex\Statex
\renewcommand{\Statex}[1][0]{%
  \setlength\@tempdima{\algorithmicindent}%
  \OldStatex\hskip\dimexpr#1\@tempdima\relax}
\makeatother

\algnewcommand\algorithmicinput{\textbf{Input:}}
\algnewcommand\Input{\item[\algorithmicinput]}
\algnewcommand\algorithmicoutput{\textbf{Output:}}
\algnewcommand\Output{\item[\algorithmicoutput]}

\newtheorem{thm}{Theorem}

\newtheorem{prop}[thm]{Proposition}

\makeatletter
\let\NAT@parse\undefined
\makeatother

\DeclareCaptionLabelSeparator{periodspace}{. }
\newcommand{\rev}[1]{{#1}}

\usepackage[numbers, sort&compress]{natbib}

\title{\Huge Discrete-Continuous Smoothing and Mapping} \author{Kevin J.
  Doherty, Ziqi Lu, Kurran Singh, and John J. Leonard%
  \thanks{This work was
    supported by ONR Neuro-Autonomy MURI grant N00014-19-1-2571 and ONR grant
    N00014-18-1-2832.}%
  \thanks{The authors are with the Computer Science and Artificial Intelligence
    Lab (CSAIL), Massachusetts Institute of Technology (MIT), Cambridge, MA
    02139. {\tt\footnotesize \{kdoherty, ziqilu, singhk,
      jleonard\}@mit.edu}}%
} %Use only for final RAL version

\usepackage{accents}  %
\newcommand{\ubar}[1]{\underaccent{\bar}{#1}}

\newcommand{\R}{\mathbb{R}}

\def \transpose{^\mathsf{T}}

\DeclareMathOperator{\Graph}{graph}

\DeclareMathOperator{\SE}{SE}

\DeclareMathOperator*{\argmax}{argmax}
\DeclareMathOperator*{\argmin}{argmin}

\def \Graph {\mathcal{G}}  %
\def \Nodes {\mathcal{V}}  %
\def \Edges {\mathcal{E}}  %
\newcommand{\directed}[1]{\vec{#1}}  %
\def \dEdges{\directed{\Edges}}  %

\def \edge{\lbrace i,j \rbrace}  %

\def \pose{x}

\newcommand{\true}[1]{\ubar{#1}}
\newcommand{\noisy}[1]{\tilde{#1}}
\newcommand{\est}[1]{\hat{#1}}

\def \tpose{\true{\pose}}

\def \optsym{*}

\def \npose{\noisy{\pose}}

\def \discreteVars{D}
\def \continuousVars{C}
\def \measurements{Z}
\def \measurement{z}
\def \Factors{\mathcal{F}}
\def \factor{f}
\def \node{v}
\def \NLL{\mathcal{L}}

\newcommand{\opt}[1]{#1^\optsym}

\def \DCSAM{DC-SAM}
\def \GTSAM{GTSAM}

\def \Intel{\textit{Intel}}
\def \CSAIL{\textit{CSAIL}}

\begin{document}

\maketitle

\begin{textblock*}{0.9\textwidth}(.05\textwidth,-4.9cm)
  \begin{center}
    This is an extended technical report for our publication in the \emph{IEEE
      Robotics and
      Automation Letters (RA-L)}. \\
    Please cite the paper as: K. Doherty, Z. Lu, K. Singh, and J. Leonard, \\
    ``Discrete-Continuous Smoothing and Mapping'', \emph{IEEE Robotics and
      Automation Letters (RA-L)}, 2022.
  \end{center}
\end{textblock*}

\begin{abstract}
  We describe a general approach \rev{for maximum \emph{a posteriori} (MAP)
    inference in} a class of \emph{discrete-continuous factor graphs} commonly
  encountered in robotics applications. While there are openly available tools
  providing flexible and easy-to-use interfaces for specifying and solving
  \rev{inference} problems formulated in terms of \emph{either} discrete
  \emph{or} continuous graphical models, at present, no similarly general tools
  exist enabling the same functionality for \emph{hybrid} discrete-continuous
  problems. We aim to address this problem. In particular, we provide a library,
  \DCSAM{}, extending existing tools for \rev{inference} problems defined in
  terms of factor graphs to the setting of discrete-continuous models. A key
  contribution of our work is a novel solver for efficiently recovering
  approximate solutions to discrete-continuous \rev{inference} problems. The key
  insight to our approach is that while joint inference over continuous and
  discrete state spaces is often hard, many commonly encountered
  discrete-continuous problems can naturally be split into a ``discrete part''
  and a ``continuous part'' that can individually be solved easily. Leveraging
  this structure, we optimize discrete and continuous variables in an
  alternating fashion. In consequence, our proposed work enables straightforward
  representation of and approximate inference in discrete-continuous graphical
  models. We also provide a method to \rev{approximate} the uncertainty in
  estimates of both discrete and continuous variables. We demonstrate the
  versatility of our approach through its application to distinct robot
  perception applications, including robust pose graph optimization, and
  object-based mapping and localization.
\end{abstract}


\section*{Supplemental Material}

\noindent The \DCSAM{} library is currently available at
\url{https://www.github.com/MarineRoboticsGroup/dcsam}.

% An extended technical
% report can be found at \cite{doherty2022discrete-techreport}.

% Several components of
% \DCSAM{} are presently undergoing migration into \GTSAM{} at
% \url{https://www.github.com/borglab/gtsam}.

\section{Introduction}

\IEEEPARstart{P}{robabilistic} graphical models have become the dominant
representational paradigm in robot perception applications, appearing in a wide
range of important estimation problems. This formalism has led to the
development of numerous algorithms and software libraries, such as \GTSAM{}
\cite{dellaert2012factor}, which provide flexible and modular languages for
specifying and solving \rev{inference} problems \rev{in} these models
\rev{(typically in terms of factor graphs)}. Among the models relevant to
robotics applications, \emph{discrete-continuous graphical models} capture a
great breadth of key problems arising in robot perception, task and motion
planning \cite[Sec 3.2]{garrett2021integrated}, and navigation, including data
association, outlier rejection, and semantic simultaneous localization and
mapping (SLAM) \cite{rosen2021advances} (see Figure \ref{fig:graph-examples}).
Despite the importance of these models, while \emph{ad hoc} solutions have been
proposed for \emph{particular} problem instances, at present there is no
off-the-shelf approach for hybrid problems that is either as general or as
easy-to-use as similar methods for their continuous-only or discrete-only
counterparts. This is the problem that we consider in this paper.

\begin{figure}[t]
  \centering
  \begin{subfigure}{1.0\linewidth}
    \centering
    \includegraphics[width=0.8\linewidth]{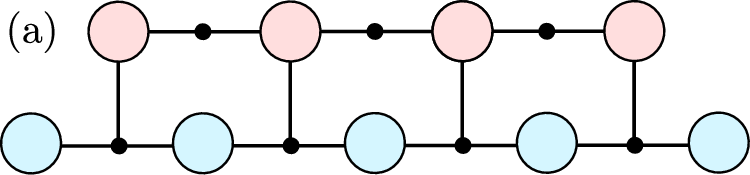}
    % \caption{\label{fig:graph-examples:switching}}
    \phantomsubcaption{\label{fig:graph-examples:switching}}
  \end{subfigure}\par\bigskip
  \begin{subfigure}{1.0\linewidth}
    \centering
    \includegraphics[width=0.8\linewidth]{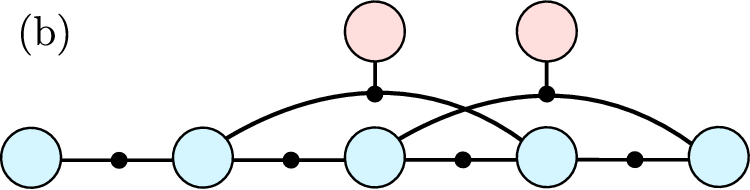}
    % \caption{\label{fig:graph-examples:rpgo}}
    \phantomsubcaption{\label{fig:graph-examples:rpgo}}
  \end{subfigure}\par\bigskip
  \begin{subfigure}{1.0\linewidth}\label{fig:graph-examples:pcr}
    \centering
    \includegraphics[width=0.8\linewidth]{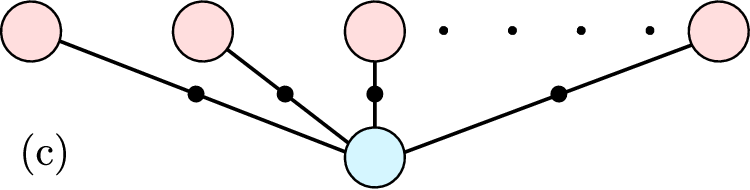}
    % \caption{\label{fig:graph-examples:pcr}}
    \phantomsubcaption{\label{fig:graph-examples:pcr}}
  \end{subfigure}
  \caption{\textbf{Discrete-continuous factor graphs in robotics.} Factor graphs
    modeling several relevant discrete-continuous robot perception problems.
    Discrete variable nodes are colored red, continuous variable nodes are blue,
    and factor nodes are black. (a) Switching systems: discrete states control
    the evolution of a continuous process. (b) Outlier rejection: discrete
    inlier/outlier variables control whether a subset of untrusted measurements
    should be used in estimating continuous variables.
    (c) Point-cloud registration: discrete variables represent correspondences
    and the continuous variable is the relative transformation from a source to
    target point-cloud. \label{fig:graph-examples}}
  % \vspace{-8.mm}
\end{figure}

Our key insight is that in many instances, while \rev{maximum \emph{a
    posteriori} (MAP) inference for graphical models containing both discrete
  and continuous variables is hard (see e.g. \cite[Sec. 14.3.1]{Koller09book})},
if we fix either the discrete or continuous variables, local optimization of the
other set is easy. Continuous optimization can be performed using smooth,
gradient-based methods, while discrete optimization can be performed
\emph{exactly} for a fixed assignment to the continuous variables by means of
standard max-product variable elimination \rev{\cite[Sec.
  13.2.1]{Koller09book}}. \rev{In turn, our approach can be perform efficient
  inference in high-dimensional, nonlinear models commonly encountered in
  robotics.} Moreover, this approach naturally extends many of the additional
desired capabilities of an inference approach in robotics applications, such as
incremental computation \cite{Kaess12ijrr} and uncertainty estimation (cf.
\cite{Kaess09ras}) to the hybrid setting.

Our contributions are as follows: From a robotics science standpoint, we show
that by leveraging the conditional independence structure of hybrid factor
graphs commonly encountered in robotics problems, efficient local optimization
can be performed using alternating optimization, which we prove guarantees
monotonic improvement in the objective. Because our approach naturally respects
the incremental structure of many such problems, it easily scales to thousands
of discrete variables \emph{without} the need to prune discrete assignments.
From a systems standpoint, our discrete-continuous smoothing and mapping
(\DCSAM{}) library extends existing \GTSAM{} tools by adding (1) explicit
constructions for hybrid discrete-continuous factors, (2) a new solver capable
of computing approximate solutions to the corresponding estimation problems, and
(3) an approach for \rev{approximating} uncertainties associated with solutions
to these problems which does not depend on the solver we employ (and therefore
is likely to be of independent interest). To the best of our knowledge, these
are the first openly-available tools for general discrete-continuous factor
graphs encountered in robotics applications. We demonstrate the application of
our methods to point-cloud registration, robust pose graph optimization, and semantic SLAM. In the former
cases, our general approach naturally recovers well-known solutions. In semantic
SLAM, our method produces high-quality trajectory and semantic map estimates
incrementally during navigation.

\section{Related Work}\label{sec:related}

The problem of inference in discrete-continuous graphical models arises in many
domains and intersects a number of communities, even within the field of
robotics. Since our focus in this paper will be on applications in robot
perception, we primarily discuss related works in these settings. \rev{The
  interested reader may refer to \citet{dellaert2021factor} for a discussion of
  these models in broader robotics applications or \citet[Ch. 14]{Koller09book}
  for a discussion of computational hardness, inference techniques, and a
  detailed review of literature on the general problem of inference in hybrid
  models.} Finally, while we discuss our alternating minimization approach in
relation to existing methods, it is important to note: the mere availability of
a consistent framework in which these solutions could be implemented enables
practitioners to compare different approaches without the need to develop the
additional scaffolding usually required to adapt an existing method.

\rev{\textbf{Multi-Hypothesis Methods.}} The class of approaches addressing
hybrid estimation by enumerating and pruning solutions to the discrete states
are referred to as \emph{multi-hypothesis methods}. These methods appeared in
classical detection and tracking problems \cite{reid1979algorithm} and early
SLAM applications \cite{cox1994modeling}. MH-iSAM2 \cite{hsiao2019mh} extends
the capabilities of iSAM2 \cite{Kaess12ijrr} to the case where measurements
between continuous variables may have ambiguity, which can be represented by the
introduction of discrete \rev{variables}. \rev{MH-iSAM2} \rev{maintains} a
hypothesis tree, which can be constructed and updated in an incremental fashion,
like iSAM2, making the solver efficient. The types of ambiguities they consider
can all be represented as factors in a factor graph where the discrete variables
are all conditionally independent. This limits application to scenarios where
individual discrete variables can be \emph{decoupled}. However, correlations
between discrete variables may arise in problem settings as diverse as switching
systems (Figure \ref{fig:graph-examples}; see also \cite{jiang2021imhs}),
outlier rejection,\footnote{Though we do not explore the issue of outlier
  rejection problems with correlations, the interested reader may see
  \citet{lajoie2019modeling} for a formulation in the setting of SLAM.} and as
we explore in Section \ref{sec:semantic-slam}, semantic SLAM. In order to retain
computational efficiency, MH-iSAM2, like all multi-hypothesis methods, must
prune hypotheses, which risks the deletion of hypotheses that would have later
become high-probability modes. iMHS \cite{jiang2021imhs} takes a qualitatively
similar approach to MH-iSAM2, \rev{but focus on the problem of smoothing in
  dynamic hybrid models, exploiting the specific temporal structure of this
  problem setting}. Their approach extends to the setting where correlations
among discrete variables are present. Like MH-iSAM2, however, the efficiency of
iMHS rests on the ability to prune incorrect modes.

\rev{\textbf{Hybrid and Non-Gaussian Inference.}} \rev{Hybrid inference in
  graphical models has been considered previously in many settings (see
  \cite{salmeron2018review} for a review). Prior solution methods focus on
  either \emph{specific} models, such as conditional linear Gaussian models
  (e.g. \cite{ramos2017map}) or attempt to approximate more general models in a
  manner amenable to standard techniques (e.g. by discretizing continuous state
  spaces to form a discrete inference problem). Models encountered in robotics
  applications are typically high-dimensional (often with numbers of states in
  the thousands) and non-Gaussian \cite{rosen2021advances}, and solutions are
  often required quickly. This precludes \emph{direct} application of these
  techniques to the problems we explore in Section \ref{sec:examples}.}

Several approaches have been presented which consider non-Gaussian inference
with application to robot perception; \rev{many of these methods can be viewed
  as adaptations of general hybrid inference techniques tailored toward the
  computational requirements and problem structure in specific robot perception
  problems}. FastSLAM \cite{montemerlo2003simultaneous} is an approach to
filtering in SLAM with non-Gaussian models based on particle filters. In
particular, a set of particles representing the current state of a robot is
retained, and each particle independently samples associations from a
distribution over hypotheses. \rev{ Multimodal iSAM (mm-iSAM)
  \cite{fourie2016nonparametric} and NF-iSAM \cite{huang2021nf} perform
  incremental non-Gaussian inference for continuous-valued variables using
  nonparametric belief propagation \cite{sudderth2010nonparametric} and
  normalizing flows, respectively. In situations where discrete variables can be
  efficiently marginalized to produce a problem exclusively involving continuous
  states, they can approximate the posterior marginals over the remaining
  continuous variables.}

\rev{In contrast, our work focuses on the task of MAP estimation from the
  perspective of local optimization. While we describe a mechanism for
  approximating marginals \emph{given} an (approximate) MAP estimate, the
  uncertainties provided by non-Gaussian inference techniques can be
  substantially richer. However, considering this somewhat more restricted
  problem setting (and coarser marginal approximation) affords us considerable
  benefits in terms of computational expense. Prior works applying optimization
  techniques for MAP estimation in non-Gaussian models (e.g.
  \cite{pfeifer2021advancing, rosen2013robust}) do so by marginalizing out
  discrete variables and using smooth local optimization techniques on the
  resulting \emph{continuous-only} estimation problem. Consequently, they do not
  permit the explicit estimation of discrete states, as we consider here.}

\rev{\textbf{Existing Tools.}} Several existing solvers perform optimization
with models that can be represented in terms of factor graphs. Ceres
\cite{ceres-manual} and g2o \cite{grisetti2011g2o} provide nonlinear
least-squares optimization tools suitable for robotics applications, but they
are not suitable for inference in \rev{hybrid} factor graphs\rev{, e.g. as in
  Figure \ref{fig:graph-examples}}. GTSAM \cite{dellaert2012factor} provides
incremental nonlinear least-squares solvers, like iSAM2 \cite{Kaess12ijrr}, and
tools for representing and solving discrete factor graphs; it is for these
reasons that we choose to extend the capabilities of GTSAM to the setting of
hybrid, discrete-continuous models. Finally, Caesar.jl \cite{Caesarjl2021}
implements mm-iSAM \cite{fourie2016nonparametric}, \rev{supporting} approximate,
incremental non-Gaussian inference over graphical models commonly encountered in
SLAM\rev{, including} discrete-continuous models in scenarios where discrete
variables can be eliminated through marginalization to produce a problem
exclusively involving continuous variables.

\section{Background and Preliminaries}\label{sec:background}

A \emph{factor graph} $\Graph \triangleq \{\Nodes, \Factors, \Edges\}$ with
factor nodes $\factor_k \in \Factors$, variable nodes $\node_i \in \Nodes$, and
edges $\Edges$ is a graphical representation of a product factorization of a
function. In our setting, we are interested in determining the most probable
assignment to a set of discrete variables $\discreteVars$ and continuous
variables $\continuousVars$ given a set of \rev{\emph{realized}} measurements $\measurements$. Under
the assumption that each measurement $\measurement_k$ is independent of all
others given the subset of variables \rev{$\Nodes_k \subseteq \Nodes$} it
relates, we can decompose the posterior $p(\continuousVars, \discreteVars \mid
\measurements)$ into a product of measurement factors $\factor_k$, each of which
depends only on a subset of variables $\Nodes_k$:
\begin{equation}
  \begin{gathered}
    p(\continuousVars, \discreteVars \mid \measurements) \propto
    \prod_k \factor_k(\Nodes_k), \\
    \Nodes_k \triangleq \{\node_i \in \Nodes \mid (\factor_k, \node_i) \in \Edges\}, \label{eq:factor-graph}
  \end{gathered}
\end{equation}
where each factor $f_k$ is in correspondence with either a measurement
likelihood of the form $p(\measurement_k \mid \Nodes_k)$ or a prior
$p(\Nodes_k)$. From \eqref{eq:factor-graph}, the posterior $p(\continuousVars,
\discreteVars \mid \measurements)$ can be decomposed into factors $f_k$ of
three possible types: \emph{discrete} factors $f_k(\discreteVars_k)$ where
$\discreteVars_k \subseteq \discreteVars$, \emph{continuous} factors
$f_k(\continuousVars_k),\ \continuousVars_k \subseteq \continuousVars$, and
\emph{discrete-continuous} factors $f_k(\continuousVars_k, \discreteVars_k)$. In
turn, the maximum \emph{a posteriori} inference problem can be posed as follows:
\begin{equation}\label{eq:map-inference}
  \begin{aligned}
   \continuousVars^*, \discreteVars^* &= \argmax_{\continuousVars, \discreteVars} p(\continuousVars, \discreteVars \mid \measurements) \\
    &= \argmax_{\continuousVars, \discreteVars} \prod_k \factor_k(\Nodes_k) \\
  &= \argmin_{\continuousVars, \discreteVars} \sum_k - \log \factor_k(\Nodes_k).
  \end{aligned}
\end{equation}
That is to say, we can maximize the posterior probability $p(\continuousVars,
\discreteVars \mid \measurements)$ by minimizing the negative log posterior,
which in turn decomposes as a summation. Though the theoretical aspects of the
methods we propose are quite general, in application we will primarily be
concerned with factor graphs in which maximum likelihood estimation (or maximum
\emph{a posteriori} inference) can be represented in terms of a nonlinear
least-squares problem, which permits the application of incremental nonlinear
least-squares solvers like iSAM2 \cite{Kaess12ijrr}.\footnote{This turns out not
  to be particularly restrictive; as demonstrated in \citet[Theorem
  1]{rosen2013robust}, any factor which is positive and bounded admits an
  equivalent representation in terms of a nonlinear least-squares cost function
  for the purposes of optimization.} In particular, we consider
discrete-continuous factors $\factor_k(\continuousVars_k , \discreteVars_k )$
admitting a description as:
\begin{equation}\label{eq:dcfactor}
  \begin{gathered}
  - \log f_k (\continuousVars_k , \discreteVars_k) = \|r_k(\continuousVars_k, \discreteVars_k)\|_2^2, \\ \continuousVars_k \subseteq \continuousVars,\ \discreteVars_k \subseteq \discreteVars,
  \end{gathered}
\end{equation}
where \rev{the function $r_k: \Omega \times \mathbb{D} \rightarrow
  \mathbb{R}^m,\ \Omega \subseteq \mathbb{R}^d,\ \mathbb{D} \subseteq
  \mathbb{N}_{0}^{|D|}$} is first-order differentiable with respect to
$\continuousVars$. We consider factors involving only continuous variables
admitting an analogous representation. We place no restriction on discrete
factors.

\begin{figure*}[t]
  \centering
  \begin{subfigure}{0.5\linewidth}
    \centering
    \includegraphics[width=0.6\columnwidth]{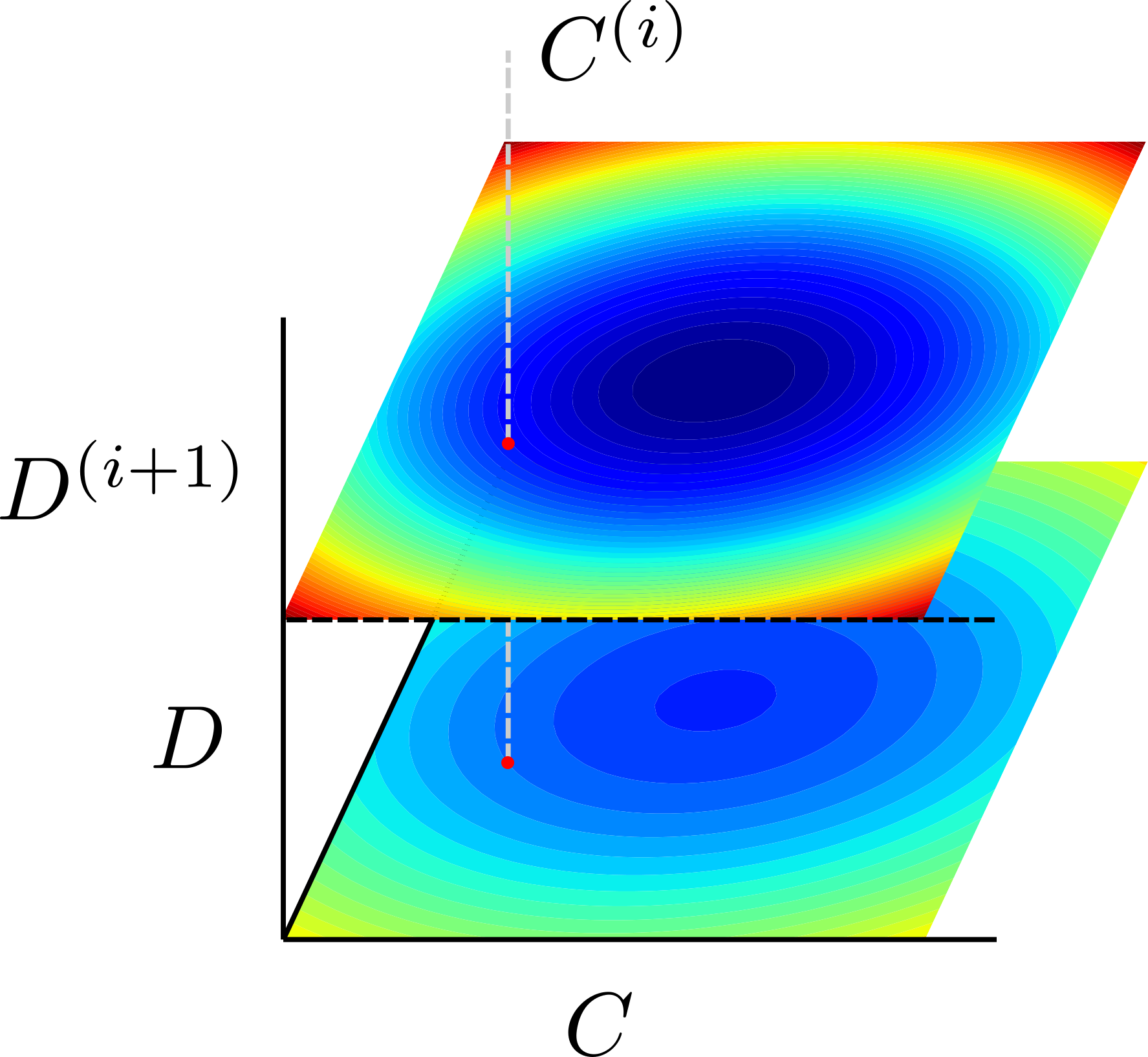}
    \caption{Discrete solve iteration.}\label{subfig:discreteSolve}
  \end{subfigure}%
  \begin{subfigure}{0.5\linewidth}
    \centering
    \includegraphics[width=0.6\columnwidth]{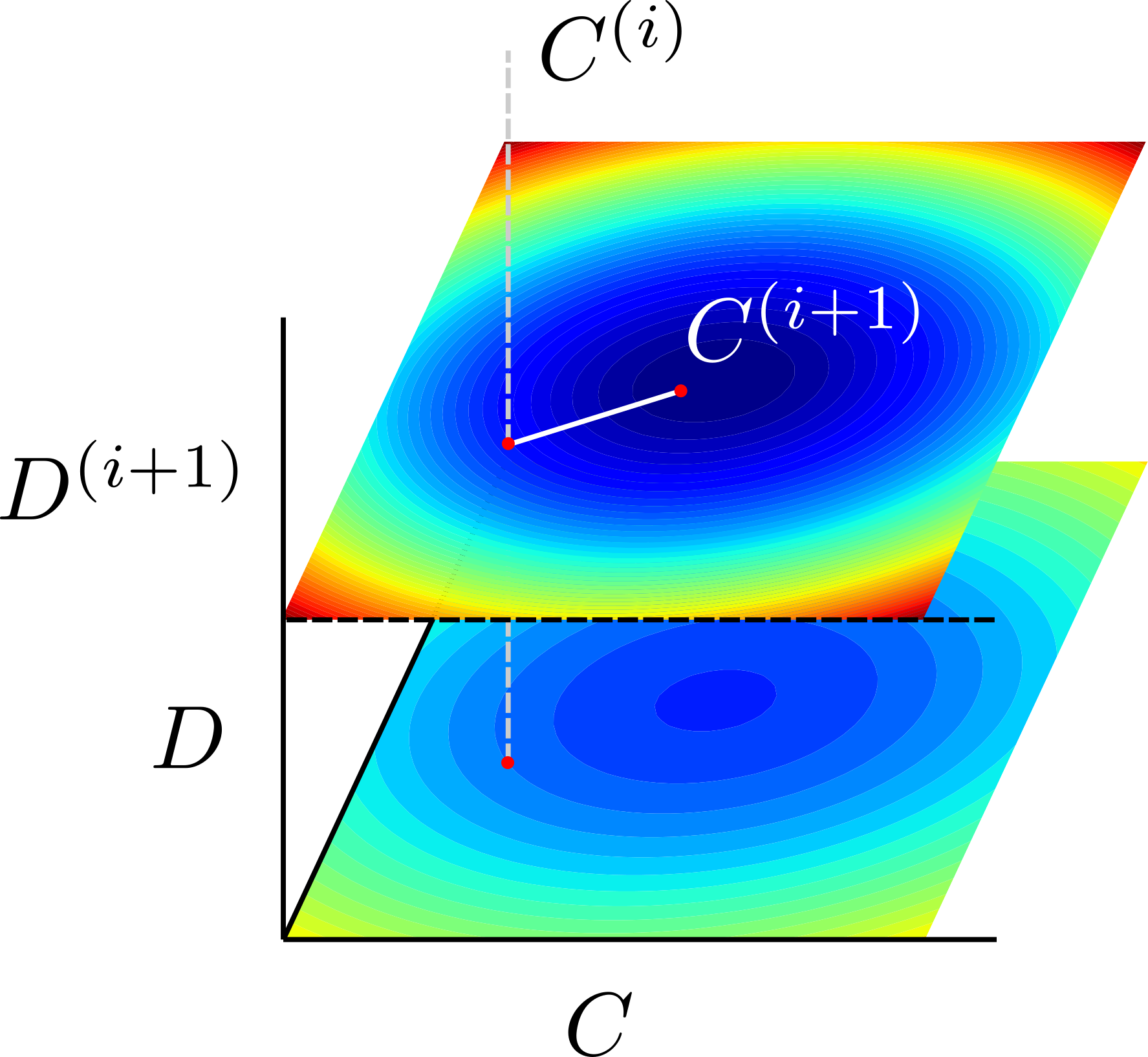}
    \caption{Continuous solve iteration.}\label{subfig:continuousSolve}
  \end{subfigure}
  \caption{\textbf{Overview of a single iteration of optimization.} (a) First,
    given an initial iterate $\continuousVars^{(i)}$ we solve exactly for the
    optimal assignment to the discrete variables using max-product elimination.
    (b) Next, given the latest assignment to the discrete variables, we update
    the continuous variables (e.g. using a trust-region method
    \cite{rosen2012incremental}). Color depicts the objective value of a
    solution, ranging from low cost (blue) to high cost
    (red).}\label{fig:optimization}
\end{figure*}

\section{Overview of the Approach}\label{sec:approach}

The following subsections describe our approach to solving optimization problems
of the form in \eqref{eq:map-inference}. In Section
\ref{sec:alternating-minimization}, we outline our core alternating minimization
procedure and prove that our approach \rev{guarantees} monotonic descent. In
Section \ref{sec:incremental}, we describe how our approach can easily benefit
from existing incremental optimization techniques to efficiently solve
large-scale estimation problems. Finally, in Section \ref{sec:marginals}, we
consider the issue of \rev{estimating} \emph{uncertainties} for the
\rev{solutions} provided by our method.

\subsection{Alternating Minimization}\label{sec:alternating-minimization}

In general, the MAP inference problem in \eqref{eq:map-inference} is
computationally intractable \rev{\cite[Sec. 13.1.1]{Koller09book}}. Indeed, even
the purely continuous estimation problems arising in robot perception are
typically NP-hard, including rotation averaging and pose-graph SLAM
\cite{rosen2021advances}. Despite this, smooth (local) optimization methods
often perform quite well on such problems, both in their computational
efficiency (owing to the fact that gradient computations are typically
inexpensive) and quality of solutions when a good initialization can be
supplied. However, even if we assume the ability to efficiently solve continuous
estimation problems, the introduction of discrete variables complicates matters
considerably: in the worst-case, solving for the joint MAP estimate globally
requires that for each assignment to the discrete states we solve a continuous
optimization subproblem, and discrete state spaces grow \emph{exponentially} in
the number of discrete variables under consideration. Consequently, efficient
approximate solutions are needed.

Our key insight is that we can leverage the conditional independence structure
of the factor graph model to develop an \rev{efficient local optimization method
  which we prove guarantees monotonic improvement in the posterior probability}.
To motivate our approach, we first observe that if we fix any assignment to the
discrete states, the only variables remaining are continuous and approximate
inference can be performed efficiently using smooth optimization techniques
\rev{\cite{Kaess12ijrr}, \cite{rosen2012incremental}}. In this sense, if we
happened to know the assignment to the discrete variables, continuous
optimization becomes ``easy.'' On the other hand, if we fix an estimate for the
continuous variables, we are left with an optimization problem defined over a
discrete factor graph which can be solved to global optimality using max-product
variable elimination \rev{\cite[Sec. 13.2.1]{Koller09book}}, but in the worst
case may still require exploration of \emph{exponentially many} discrete states.
However, it turns out that for many commonly encountered problems, we can often
do much better than the worst case.

\rev{For example}, consider a partition of the discrete states into mutually
exclusive subsets $\discreteVars_j \subseteq \discreteVars$ which are
\emph{conditionally independent} given the continuous states:
\begin{equation}\label{eq:discrete-conditional}
  p(\discreteVars \mid \continuousVars, \measurements) \propto \prod_j p(\discreteVars_j \mid \continuousVars, \measurements).
\end{equation}
It is straightforward to verify from the mutual exclusivity of each set
$\discreteVars_j$ that the problem of optimizing the conditional in
\eqref{eq:discrete-conditional} then breaks up into subproblems involving each
$\discreteVars_j$:
\begin{equation}
  \max_{\discreteVars} p(\discreteVars \mid \continuousVars, \measurements) \propto \prod_j \left[ \max_{\discreteVars_j} p(\discreteVars_j \mid \continuousVars, \measurements) \right].
\end{equation}
Critically, we have exchanged computation of the maximum of the product with the
product of each maximum computed \emph{independently}. In cases where the
discrete states decompose into particularly small subsets ($|\discreteVars_j|
\ll |\discreteVars|$), inference may be carried out efficiently. Many hybrid
optimization problems encountered in robotics admit such advantageous
conditional independence structures. For example, Figures
\ref{fig:graph-examples:rpgo} and \ref{fig:graph-examples:pcr}, depicting robust
pose graph optimization and point-cloud registration, respectively, admit a
decomposition of the form in equation \eqref{eq:discrete-conditional} where each
subset $\discreteVars_j$ contains only \emph{a single discrete variable}.
Moreover, some discrete factor graphs do not decompose quite so drastically
after conditioning on continuous states, but may still permit efficient
inference. For example, Figure \ref{fig:graph-examples:switching} depicts a
switching system in which, after conditioning on the continuous variables, the
resulting discrete graph is a hidden Markov model, for which the most probable
assignment to the discrete states can be computed in polynomial time using the
Viterbi algorithm \cite{viterbi1967error}.

In turn, we will use these ideas to construct an algorithm for efficiently
producing solutions to problems of the form in
\eqref{eq:map-inference}.\footnote{\rev{The approach we present does not
    \emph{require} that a model admit a conditional factorization like the one
    in equation \eqref{eq:discrete-conditional}, though it improves computational
    efficiency considerably (see Section \ref{sec:efficient} for a
    discussion).}} Consider the negative log posterior, defined as:
\begin{equation}
\NLL(\continuousVars, \discreteVars) \triangleq -\log p(\continuousVars, \discreteVars \mid \measurements).
\end{equation}
From \eqref{eq:map-inference}, the joint optimization problem of interest can be formulated as:
\begin{equation}
  \opt{\continuousVars}, \opt{\discreteVars} = \argmin_{\continuousVars, \discreteVars} \NLL(\continuousVars, \discreteVars).
\end{equation}
Our alternating minimization approach (depicted in Figure
\ref{fig:optimization}) proceeds as follows: first, fix an initial
iterate $\continuousVars^{(i)}$. Then, we aim to solve the following
subproblems:
\begin{subequations}
  \begin{equation}
  \discreteVars^{(i+1)} = \argmin_{\discreteVars} \NLL(\continuousVars^{(i)}, \discreteVars) \label{eq:discrete-opt}
  \end{equation}
  \begin{equation}
    \continuousVars^{(i+1)} = \argmin_{\continuousVars} \NLL(\continuousVars, \discreteVars^{(i+1)}). \label{eq:contin-opt}
  \end{equation}
\end{subequations}
We may then repeat \eqref{eq:discrete-opt} and \eqref{eq:contin-opt} until the
relative decrease in $\NLL(\continuousVars, \discreteVars)$ is sufficiently
small or we have reached a maximum desired number of iterations. Finding
minimizers for the subproblems \eqref{eq:discrete-opt} and \eqref{eq:contin-opt}
may still be challenging. Fortunately, one need not find a minimizer for the
subproblems \eqref{eq:discrete-opt} and \eqref{eq:contin-opt} in order for our
approach to \rev{ensure} monotonic improvements to the objective. \rev{In
  particular, we require only that at each iteration the following descent
  criteria hold:}
\begin{subequations}
  \begin{equation}
    \NLL(\continuousVars^{(i)}, \discreteVars^{(i+1)}) \leq \NLL(\continuousVars^{(i)}, \discreteVars^{(i)}) \label{eq:discrete-descent}
  \end{equation}
  \begin{equation}
    \NLL(\continuousVars^{(i+1)}, \discreteVars^{(i+1)}) \leq \NLL(\continuousVars^{(i)}, \discreteVars^{(i+1)}). \label{eq:contin-descent}
  \end{equation}
\end{subequations}
\rev{There are many methods for updating the discrete and continuous states that
  satisfy \eqref{eq:discrete-descent} and \eqref{eq:contin-descent},
  respectively. For the discrete states, the descent criterion in
  \eqref{eq:discrete-descent} can be ensured by using the max-product algorithm
  to compute the \emph{optimal solution} to the subproblem in
  \eqref{eq:discrete-opt}. For the continuous states, the descent criterion in
  \eqref{eq:contin-descent} can be guaranteed by, for instance, using a
  \rev{trust region} method \rev{(e.g. \cite{rosen2012incremental})} to refine
  the continuous states with respect to the objective in \eqref{eq:contin-opt}
  .} In turn, we obtain the following proposition:
\begin{prop}\label{prop:monotonic-cost-reduction}
  Let $\NLL(\continuousVars, \discreteVars)$ be the objective to be minimized,
  with initial iterate $\continuousVars^{(0)}, \discreteVars^{(0)}$. Suppose
  that at each iteration, the discrete update satisfies the descent criterion in
  \eqref{eq:discrete-descent} and likewise for the continuous update in
  \eqref{eq:contin-descent}. Then, the estimates $\continuousVars^{(i)},
  \discreteVars^{(i)}$ obtained by alternating optimization satisfy:
  \begin{equation}
    \NLL(\continuousVars^{(0)}, \discreteVars^{(0)}) \geq \NLL(\continuousVars^{(1)}, \discreteVars^{(1)}) \geq \ldots \geq \NLL(\continuousVars^{(T)}, \discreteVars^{(T)}),
  \end{equation}
  i.e., this procedure monotonically improves the objective.
\end{prop}
\begin{proof}
  Fix an initial iterate $(\continuousVars^{(i)}, \discreteVars^{(i)})$. By
  hypothesis, after a discrete update, we have $\NLL(\continuousVars^{(i)},
  \discreteVars^{(i+1)}) \leq \NLL(\continuousVars^{(i)}, \discreteVars^{(i)})$
  (from \eqref{eq:discrete-descent}). Consequently, the updated assignment
  comprised of the pair $(\continuousVars^{(i)}, \discreteVars^{(i+1)})$ is
  \emph{at least as good} as the previous assignment. By the same reasoning,
  performing a subsequent continuous update gives a pair
  $(\continuousVars^{(i+1)}, \discreteVars^{(i+1)})$ satisfying
  $\NLL(\continuousVars^{(i+1)}, \discreteVars^{(i+1)}) \leq
  \NLL(\continuousVars^{(i)}, \discreteVars^{(i+1)})$ (from
  \eqref{eq:contin-descent}). Combining these inequalities, we have:
  \begin{equation}
    \NLL(\continuousVars^{(i+1)}, \discreteVars^{(i+1)}) \leq \NLL(\continuousVars^{(i)}, \discreteVars^{(i+1)}) \leq \NLL(\continuousVars^{(i)}, \discreteVars^{(i)}).
  \end{equation}
  The above chain of inequalities holds for all $i$, completing the proof.
\end{proof}

\subsection{Online, Incremental Inference}\label{sec:incremental}

Many robotics problems naturally admit \emph{incremental} solutions wherein new
information impacts only a small subset of the states we would like to estimate.
Because our alternating minimization approach relies only upon the ability to
provide an improvement in each of the \emph{separate} discrete and continuous
subproblem steps, we can rely on existing techniques to solve these problems in
an incremental fashion. In particular, in the continuous optimization
subproblem, we use iSAM2 \cite{Kaess12ijrr} to refactor the graph containing
continuous variables into a \emph{Bayes tree}, permitting incremental inference
of the continuous variables. Similarly, owing to the discrete factorization in
\eqref{eq:discrete-conditional}, if, for example, we introduce new discrete
variables which are conditionally independent of all previous discrete states
given the current continuous state estimate, we are able to solve for the most
probable assignment to these variables \emph{without} the need to recompute
solutions for previously estimated variables. In turn, we are able to
efficiently solve online inference problems, as we demonstrate in Section
\ref{sec:semantic-slam}, in which we produce solutions to online SLAM problems.

\subsection{Recovering Marginals}\label{sec:marginals}

Uncertainty representation is important in many applications of robot
perception. \DCSAM{} supports \rev{\emph{post hoc}} recovery of approximate
marginal distributions for discrete and continuous variables \rev{from an
  estimate}. For continuous variables, we use the \emph{Laplace approximate}
\cite[Sec. 4.4]{bishop_prml} adopted by several nonlinear least-squares solvers
(Ceres, g2o, and \GTSAM{}). In particular, we fix a linearization point for the
continuous variables (and a current estimate for discrete variables) and compute
an approximate linear Gaussian distribution centered at this linearization
point. For discrete variables, we fix an assignment to the continuous variables
and compute the exact discrete marginals conditioned on this linearization point
using \rev{clique tree propagation \cite[Ch.
  10]{Koller09book}}. The marginals
we recover, then are:
\begin{subequations}
  \begin{equation}\label{eq:discrete-marg}
    p(\discreteVars_j \mid \est{\continuousVars}, \measurements) = \sum_{\discreteVars \setminus \discreteVars_j} p(\discreteVars \mid \est{\continuousVars}, \measurements),\quad \discreteVars_j \subseteq \discreteVars,
  \end{equation}
  \begin{equation}\label{eq:continuous-marg}
    p(\continuousVars_j \mid \est{\discreteVars}, \measurements) = \int_{\continuousVars \setminus \continuousVars_j} p(\continuousVars \mid \est{\discreteVars}, \measurements),\quad \continuousVars_j \subseteq \continuousVars.
  \end{equation}
\end{subequations}
The reason for this approach is that in general, the number of posterior modes
captured by a particular (discrete-continuous) factor graph can grow
combinatorially. \rev{Computing} exact marginals can easily become intractable.
In contrast, by making use of the conditional factorization in
\eqref{eq:discrete-conditional}, solving for the discrete marginals in
\eqref{eq:discrete-marg} is often tractable.\footnote{It is also interesting to
  note that the discrete marginals we recover are \emph{exactly} the ``weights''
  computed in the expectation step of the well-known
  \emph{expectation-maximization} (EM) algorithm \cite{dempster1977maximum}.}
Notably, our approach to marginal recovery does not require that one use the
alternating minimization strategy outlined in Section
\ref{sec:alternating-minimization}; any method of providing an estimate
$(\est{C}, \est{D})$ will suffice.

The continuous marginals in \eqref{eq:continuous-marg} are estimated using the
Laplace approximation \cite{Kaess09ras}. In our derivation, it will be
convenient to consider the continuous states as a vector $\continuousVars \in
\R^d$. Assume the point $(\est{\continuousVars}, \est{\discreteVars})$ is a
critical point of the continuous subproblem \eqref{eq:contin-opt}, i.e. $\nabla
\NLL(\continuousVars, \est{\discreteVars})\vert_{\est{\continuousVars}} = 0$.
Consider a Taylor expansion of the objective $\NLL(\continuousVars,
\est{\discreteVars})$ about the point $\est{\continuousVars}$:
\begin{equation}\label{eq:nll-taylor}
  \NLL(\continuousVars, \est{\discreteVars}) \approx \NLL(\est{\continuousVars}, \est{\discreteVars}) - \frac{1}{2}A\left( \continuousVars - \est{\continuousVars} \right),
\end{equation}
with the $d \times d$ Hessian matrix $A$ defined as:
\begin{equation}
  A \triangleq - \nabla^2 \NLL(\continuousVars, \est{\discreteVars})\vert_{\est{\continuousVars}}.
\end{equation}
Exponentiating both sides of \eqref{eq:nll-taylor} and appropriately normalizing
the result gives the linear Gaussian approximation:
\begin{equation}\label{eq:laplace-marg}
  p(\continuousVars \mid \est{\discreteVars}, \measurements) \approx \frac{|A|^{1/2}}{(2\pi)^{d/2}}\exp\left\{ -\frac{1}{2} \|C - \est{C}\|^2_{A^{-1}} \right\},
\end{equation}
where $\|c\|_{A^{-1}}$ denotes the Mahalanobis norm $\sqrt{c\transpose A c}$.
When all factors involving continuous variables take the form in
\eqref{eq:dcfactor}, the locally linear approximation of $\NLL$ about $\est{C}$
admits a Hessian $A$ which can be expressed in terms of the Jacobian of the
measurement function $r$, and we have $A \succeq 0$ \cite{dellaert2017factor}.
Additionally, the relevant components of the matrix $A$ for estimating the
marginals for a subset of variables $\continuousVars_j$ can be recovered from
its square root, i.e. the \emph{square-root information matrix} (cf.
\cite{Kaess09ras}).

\section{Example Applications}\label{sec:examples}

\begin{figure*}
  \centering
  \begin{subfigure}{1.0\linewidth}
    \includegraphics[width=1.0\columnwidth]{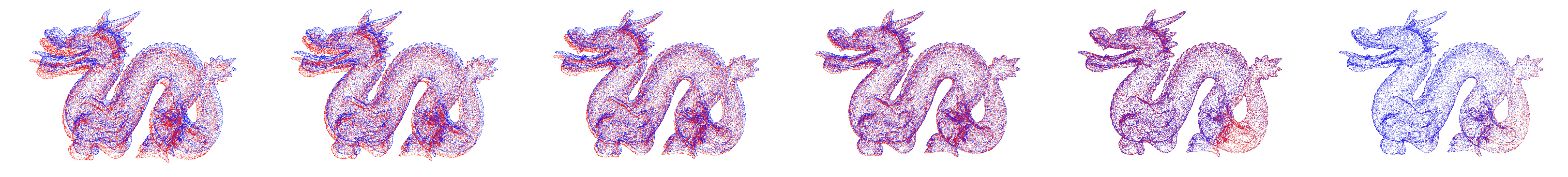}
    \caption{Point-cloud registration\label{fig:example-applications:pcr}}
    \vspace{5mm}
  \end{subfigure}
  \begin{subfigure}{1.0\linewidth}
    \includegraphics[width=1.0\columnwidth]{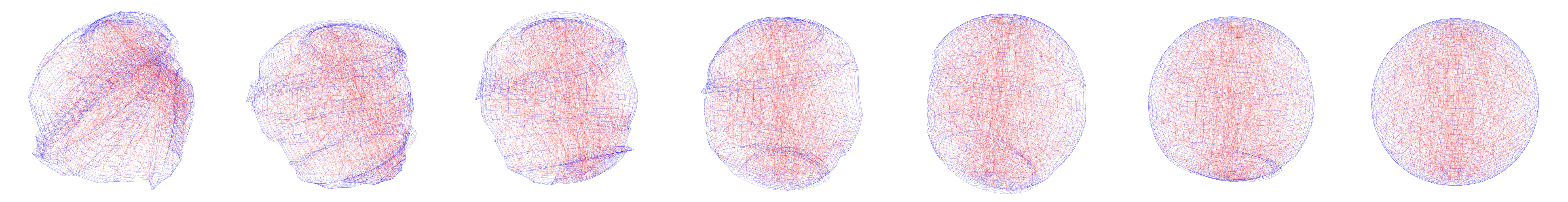}
    \caption{Robust pose graph optimization\label{fig:example-applications:rpgo}}
  \end{subfigure}
  \caption{\textbf{Example applications.} (a) Point-cloud registration using the
    \emph{Stanford Dragon} dataset \cite{curless1996volumetric}. (b) Robust pose
    graph optimization using the \emph{Sphere} dataset \cite{Kaess12ijrr}. Each
    row displays the sequence of iterates for our method. In each case, we
    obtain high-quality solutions in just a few
    iterations. \label{fig:example-applications}}
\end{figure*}

In the following sections we provide example applications motivated by typical
robot perception problems. In Section \ref{sec:icp}, we begin with an
instructive example formulating the classical problem of \emph{point-cloud
  registration} in terms of a discrete-continuous factor graph, which can be
optimized using our solver. In Section \ref{sec:rpgo}, we consider the problem
of \emph{robust pose graph optimization}, \rev{where} we aim to estimate a set
of poses given only noisy measurements between a subset of them, and some
fraction of those measurements may be outliers. We implement a straightforward
approach to solving this problem using \DCSAM{} and show that it produces
competitive results. In Section \ref{sec:semantic-slam}, we demonstrate the
application of our solver to a tightly-coupled semantic SLAM problem, where the
variables of interest are robot poses, landmark locations, and semantic classes
of each landmark.

\subsection{Point-cloud registration}\label{sec:icp}

As a simple first example, we will consider the point-cloud registration
problem. Consider a source point-cloud $\mathcal{P}_S = \{p^S_i \in \R^d,\ i =
1, \ldots, n\}$ and target point-cloud $\mathcal{P}_T = \{p^T_j \in \R^d,\ j =
1, \ldots, m\}$. Associate with each point in the source cloud $p^S_i$ a
discrete variable $d_i \in \{1, \ldots, m\}$ determining the corresponding point
in the target cloud. The goal of point-cloud registration is to identify the
rigid-body transformation $T \in \SE(3)$ that minimizes the following objective:
\begin{equation}
  \min_{T \in \SE(3)} \sum_{i=1}^n \|Tp^S_i - p^T_{d_i}\|_2^2. \label{eq:pcr}
\end{equation}
The key challenge encountered in this setting is that the correspondence
variables $d_i$ are unknown and unobserved. We might consider, then, introducing
the correspondence variables into the optimization, to determine the \emph{best}
set of correspondence variables \emph{and} the corresponding rigid-body
transformation of the point-cloud, obtaining the following problem:
\begin{equation}
  \min_{d_i \in 1:m, T \in \SE(3)} \sum_{i=1}^n \|Tp^S_i - p^T_{d_i}\|_2^2. \label{eq:pcr-hybrid}
\end{equation}
Unfortunately, this problem is nonconvex and solving it to global optimality is,
in general, NP-hard, requiring search over $\mathcal{O}(n^m)$ discrete state assignments.

A popular algorithm for solving the problem in equation \eqref{eq:pcr-hybrid} is
to first posit an initial guess for the transformation $T$, determine the
transformed locations of each of the points in the source cloud, then associate
each point in the source cloud with the \emph{nearest} point in the target cloud
after the transformation. This is the \emph{iterative closest point} (ICP)
algorithm \cite{besl1992method, chen1992object}. Defining $r_i(T, d_i) = Tp^S_i
- p^T_{d_i}$, we can see that the problem in equation \eqref{eq:pcr-hybrid} is
concisely described in terms of factors of the form \eqref{eq:dcfactor}.
Moreover, the conditional independence structure of the graph corresponding to
this problem (depicted in Figure \ref{fig:graph-examples}) immediately motivates
our alternating optimization approach, since each $d_i$ in fact \emph{decouples}
when conditioned on $T$. Finally, one can verify that our alternating
optimization procedure turns out to be identical to ICP (as described above) in
this setting. To demonstrate this fact, we applied our method to point cloud
registration using the \emph{Stanford Dragon} dataset
\cite{curless1996volumetric}, the results of which are depicted in Figure
\ref{fig:example-applications:pcr}. Indeed, we observe that our approach
produces qualitatively reasonable results in just a few iterations. Moreover,
while implementing ICP typically requires that we explicitly write the
(independent) correspondence updates and transform update, we need not encode
this explicitly at all: the fact that the discrete (correspondence) update
separates into independent subproblems is simply a consequence of the
conditional independence structure of the factor graph model in Figure
\ref{fig:graph-examples:pcr}. That said, our approach does not have knowledge
about the particular \emph{spatial} structure of the problem and therefore
performs na\"ive search over discrete assignments. In contrast, a typical
implementation of ICP would make use of efficient spatial data structures to
speed up the solution to the discrete subproblem, see
\cite{rusinkiewicz2001efficient} (indeed, such optimizations for
\emph{particular} problems like this would make for interesting future
applications of the \DCSAM{} library). However, unlike any \emph{particular} ICP
implementation, our solver can be readily extended (without modification) to
more complex cost functions or models because the structure of the subproblems
is dictated by the independence structure inherent in the graphical model.

\subsection{Robust Pose Graph Optimization}\label{sec:rpgo}

\begin{figure*}
  \centering
  \begin{subfigure}{1.0\linewidth}
    \centering
    \includegraphics[width=0.28\columnwidth]{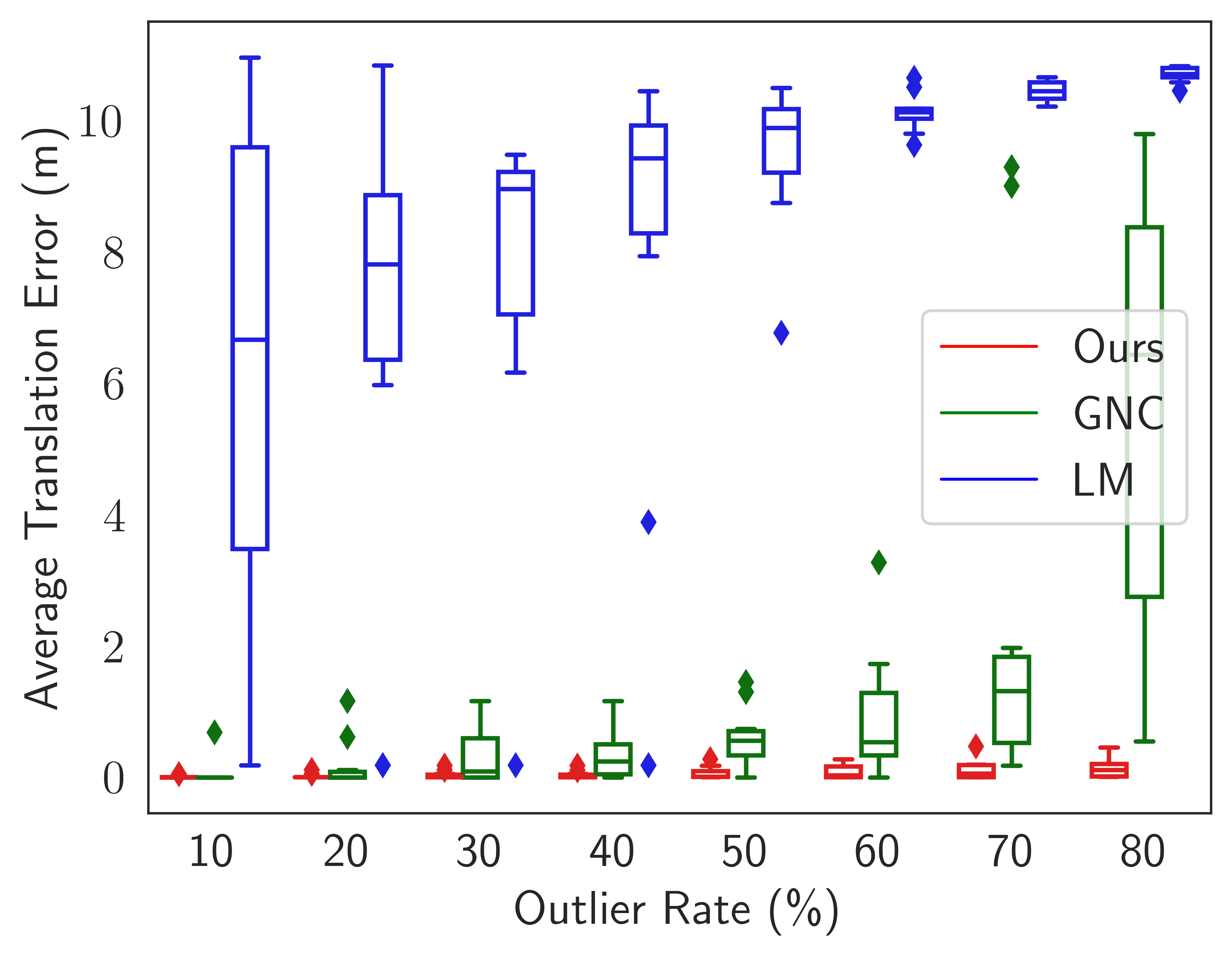}
    \includegraphics[width=0.28\columnwidth]{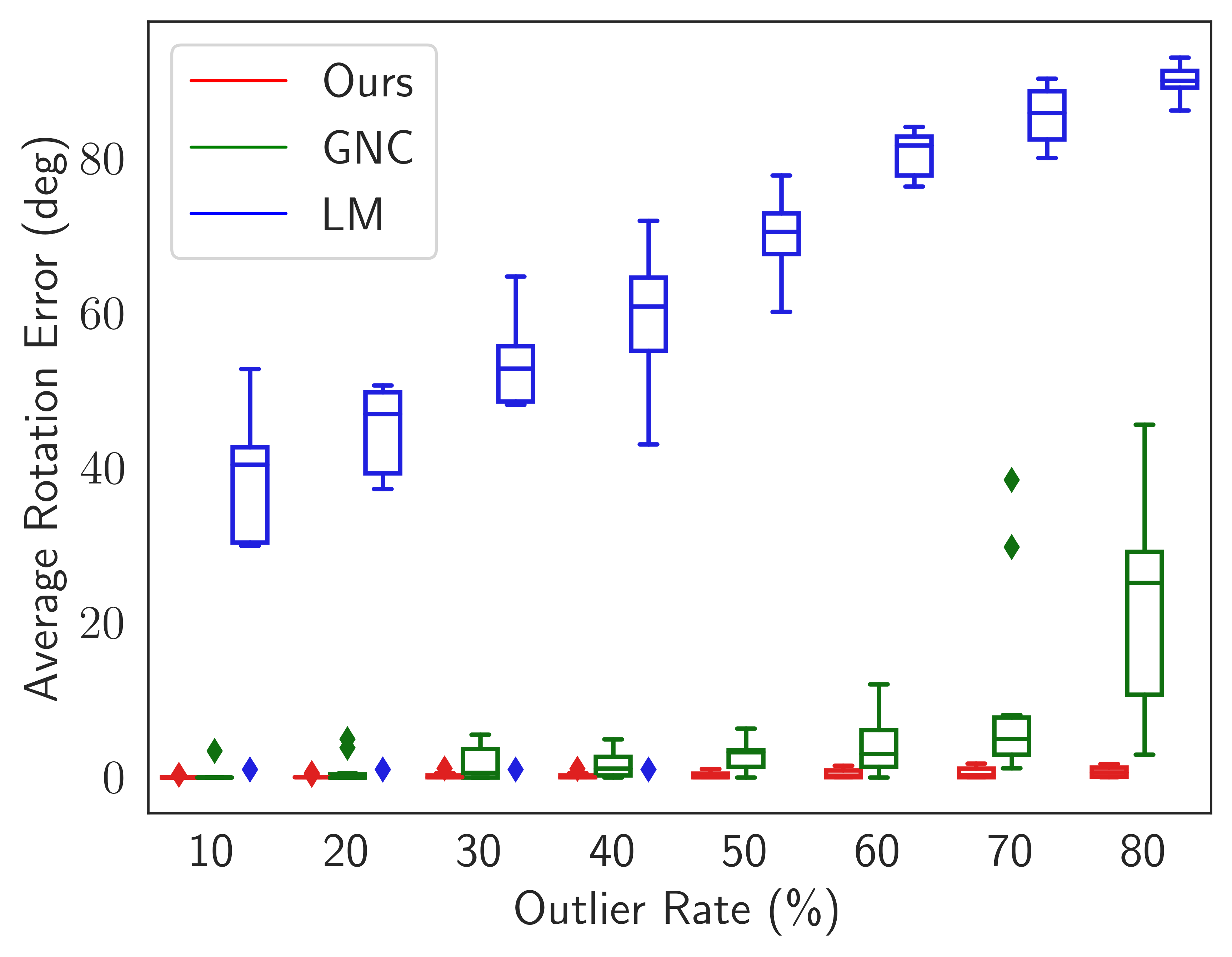}
    \includegraphics[width=0.28\columnwidth]{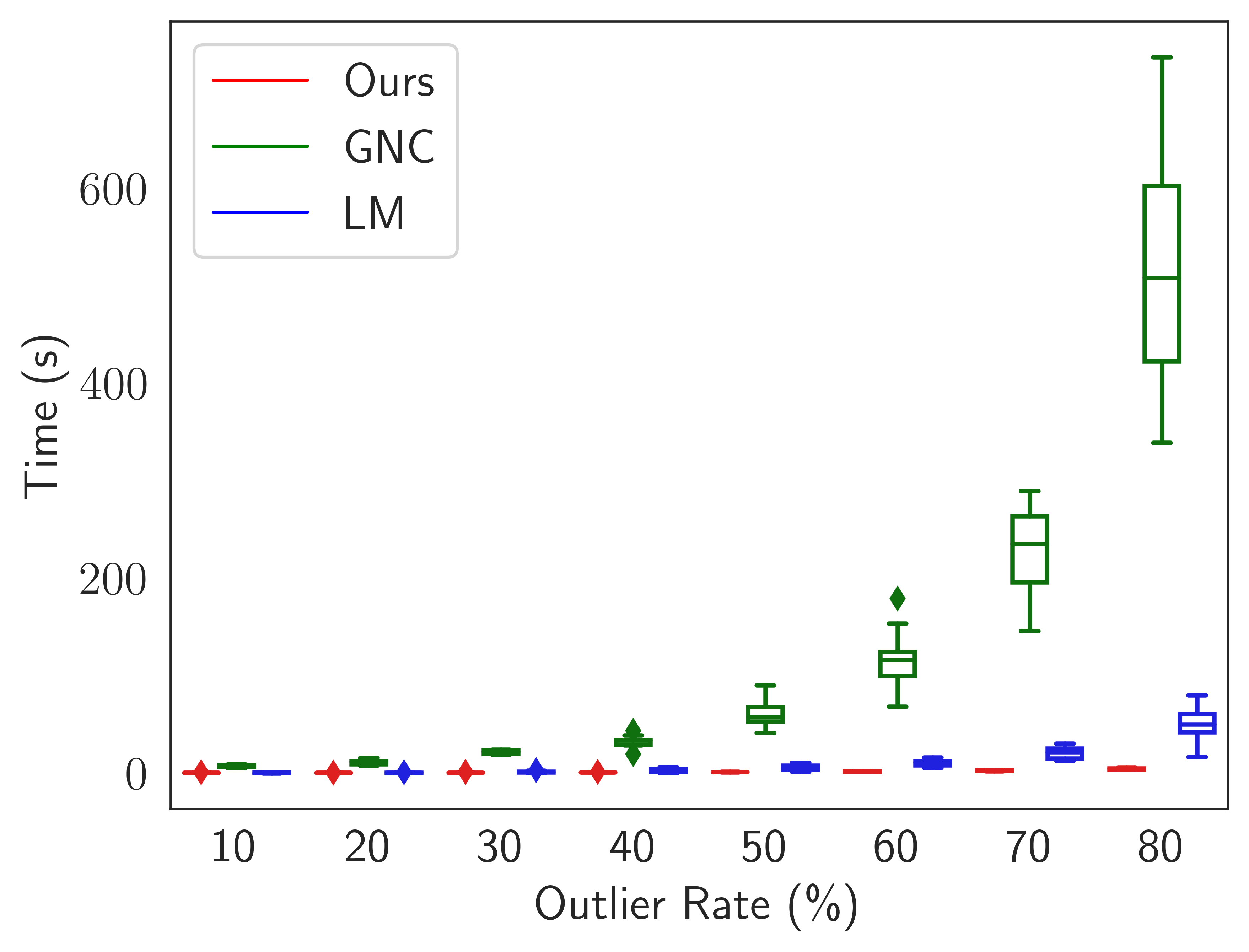}
    \caption{\Intel{}\label{pgo:intel}}
  \end{subfigure}
  \begin{subfigure}{1.0\linewidth}
    \centering
    \includegraphics[width=0.28\columnwidth]{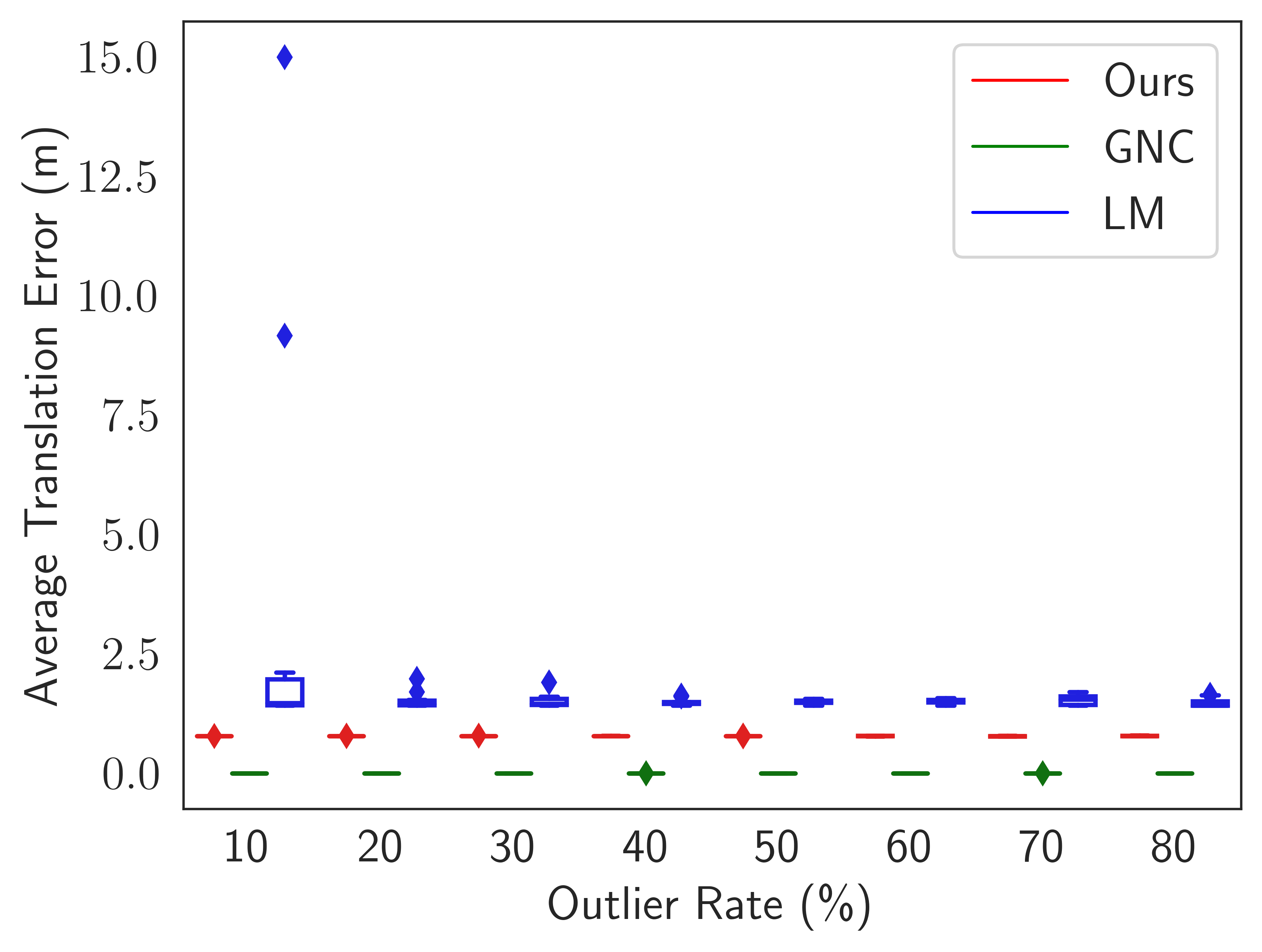}
    \includegraphics[width=0.28\columnwidth]{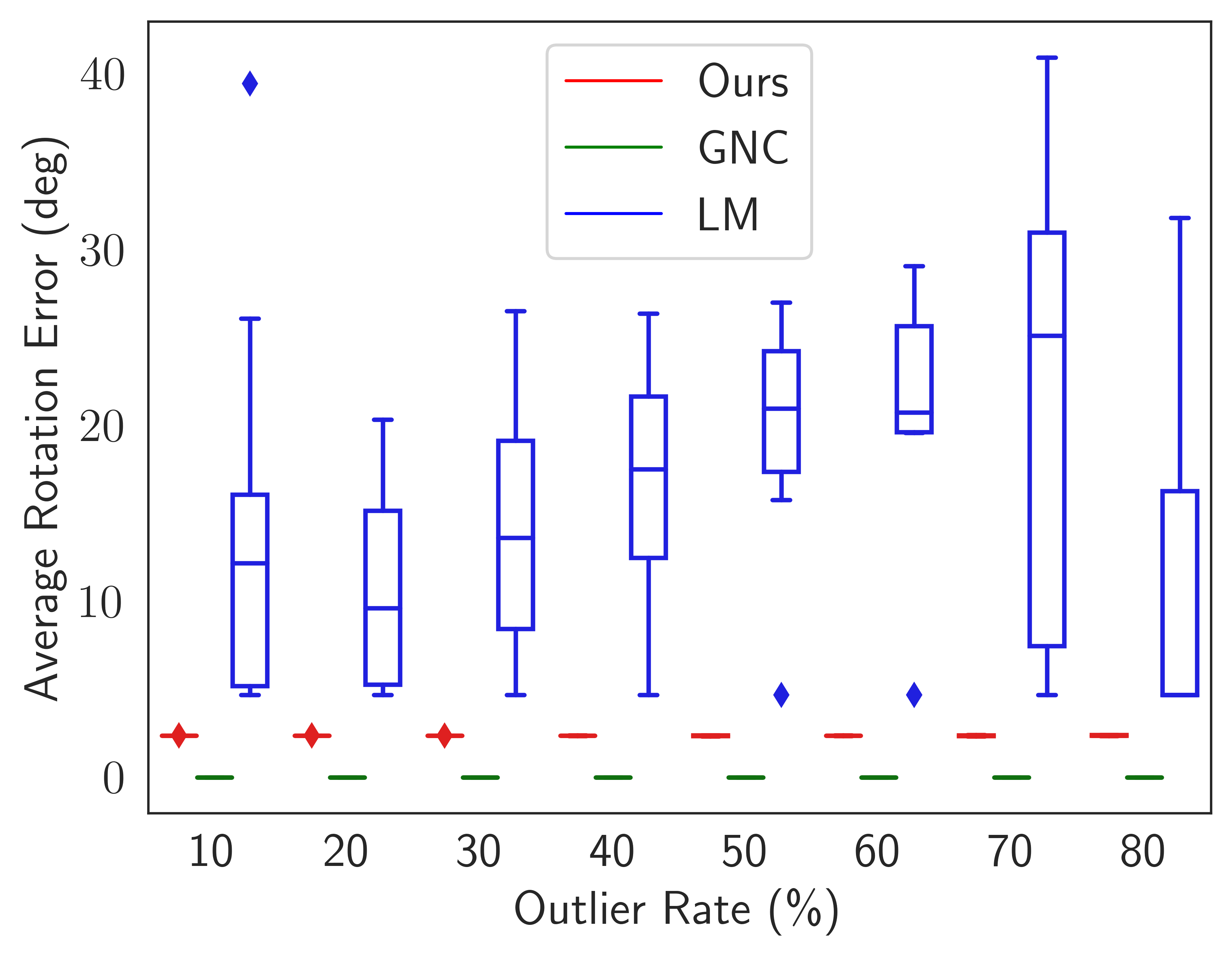}
    \includegraphics[width=0.28\columnwidth]{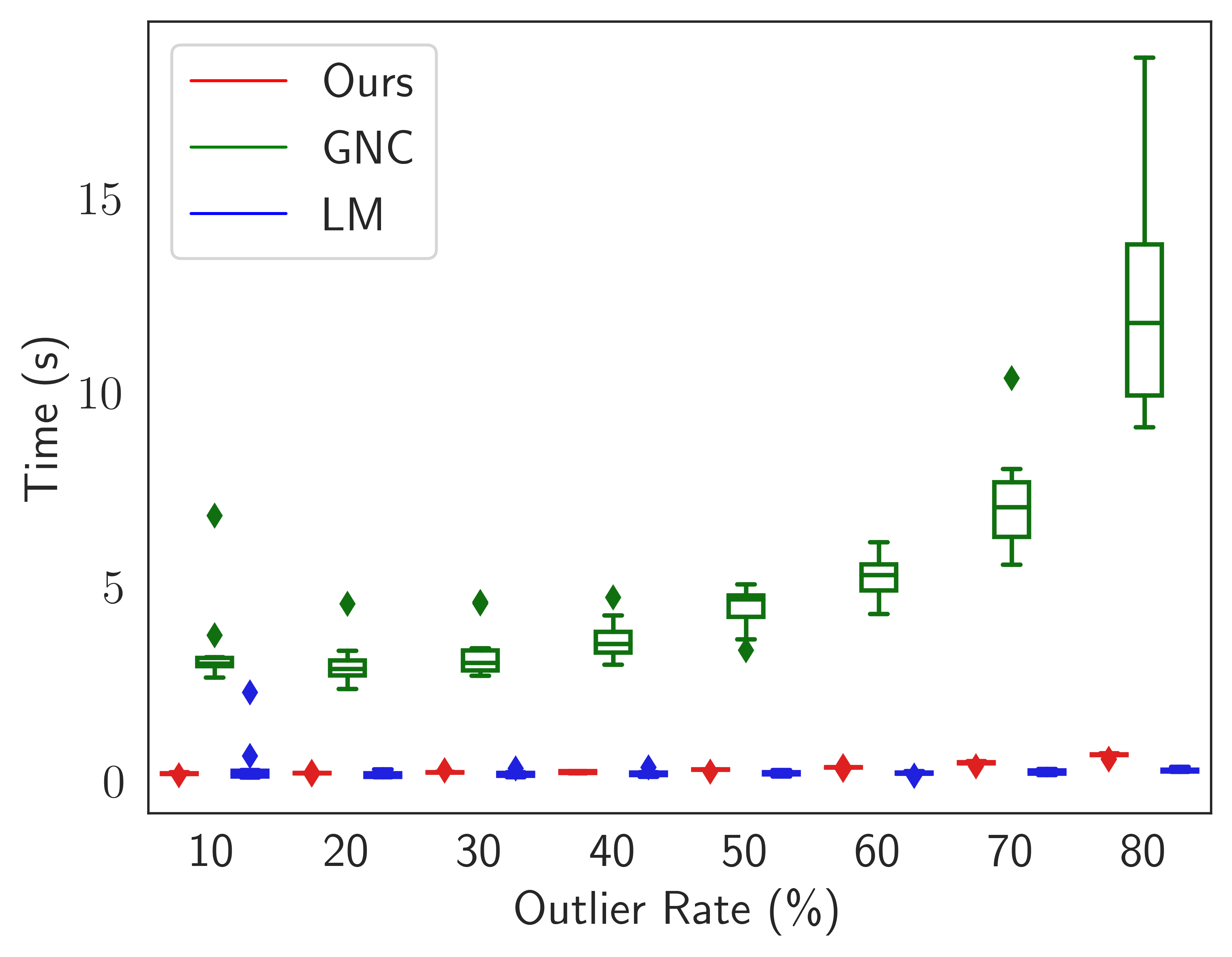}
    \caption{\CSAIL{}\label{pgo:csail}}
  \end{subfigure}
  \caption{\textbf{Robust pose graph optimization.} Average trajectory errors on
    (a) the \Intel{} dataset, (b) the \CSAIL{} dataset. Left to right:
    translation error, rotation error, and computation time. Statistics computed
    over 10 Monte Carlo trials. LM refers to the result obtained by running
    Levenberg-Marquardt on the corrupted graph.\label{fig:rpgo}}
\end{figure*}

In this section we consider \emph{robust} pose graph optimization. In pose graph
optimization we are interested in estimating a set of poses $x_1, \ldots, x_n
\in \SE(3)$ from noisy measurements $\npose_{ij}$ of a subset of their (true)
relative transforms $\tpose_{ij} = \tpose_i^{-1}\tpose_j$. This problem
possesses a natural graphical structure $\Graph = \{\Nodes, \dEdges\}$ where
nodes correspond to the poses $\pose_i$ to be estimated and edges correspond to
the available noisy measurements between them. Pose graph optimization then aims
to solve the following problem:
\begin{equation}\label{eq:pgo}
  \min_{\pose_i \in \SE(3)} \sum_{\edge \in \dEdges} \big\|\underbrace{\rev{\log\left( \npose_{ij}^{-1} \pose_i^{-1}\pose_j \right)^{\vee}}}_{r_{ij}(\pose_i, \pose_j)}\big\|^2_{\Sigma},
\end{equation}
where \rev{$\log(\cdot)^{\vee}: \SE(3) \rightarrow \R^6$ takes an element of
  $\SE(3)$ to an element of the tangent space (cf. \cite[Sec.
  8.3.2]{barfoot2017state}), and $\Sigma \in \R^{6 \times 6}$ is a covariance
  matrix}.

Suppose however, that some fraction of our measurements are corrupted by an
unknown outlier process. We would like to determine the subset of outlier
measurements and inlier measurements, as well as the corresponding optimal
poses. It is typical to assume that the edges $\dEdges$ partition into a set of
trusted odometry edges $\dEdges_{\mathcal{O}}$ and a set of untrusted loop
closure edges $\dEdges_{\mathcal{L}}$. It is common to address this problem by
introducing binary variables $d_{ij} \in \{0,1\}$ for each of the untrusted
edges (cf.
\cite{olson2013inference,sunderhauf2012switchable,agarwal2013robust, Segal2014hybrid}), where
$d_{ij} = 1$ indicates that the measurement $\npose_{ij}$ is drawn from the
outlier process. Since the outlier distribution is unknown, it is common to
assume that the outlier generating process is Gaussian with covariance
$\noisy{\Sigma} \succ \Sigma$ much larger than the inlier model covariance. In
turn, the problem of interest can be posed as follows:
\begin{equation}\label{eq:robust-pgo}
    \min_{\pose_i \in \SE(3)} \sum_{\edge \in \dEdges_{\mathcal{O}}} \|r_{ij}(\pose_i, \pose_j)\|_{\Sigma}^2 + \sum_{\edge \in \dEdges_{\mathcal{L}}} e_{ij}(\pose_i, \pose_j, d_{ij}),
\end{equation}
where
\begin{equation}
  e_{ij}(\pose_i, \pose_j, d_{ij}) \triangleq \begin{cases} -\log \omega_0 + \|r_{ij}(\pose_i, \pose_j)\|^2_{\Sigma},\ d_{ij} = 0, \\
    - \log \omega_1 + \|r_{ij}(\pose_i, \pose_j)\|^2_{\noisy{\Sigma}},\ d_{ij} = 1,
  \end{cases}
\end{equation}
and $\omega_0,\ \omega_1 \in [0,1]$ are prior weights on the inlier and outlier
hypotheses, respectively. Letting $|\dEdges_{\mathcal{L}}| = m$, there are
$\mathcal{O}(2^m)$ possible assignments to the discrete variables in this
problem. However, the above formulation can easily be represented in terms of
discrete factors for the weights $\omega_0,\ \omega_1$ and discrete-continuous
factors of the form in \eqref{eq:dcfactor} to switch between the Gaussian inlier
and outlier hypotheses. Moreover, once again, the discrete variables decouple
from one another conveniently when we condition on an assignment to the
continuous variables \rev{(Fig. \ref{fig:graph-examples}b shows the
  corresponding graph)}.

In our experimental setup, we corrupt pose graphs with outliers generated
between a random pair of (non-adjacent) poses with relative translation sampled
uniformly from a cube of side-length 10 meters and rotation sampled from the
uniform distribution over rotations (a similar process to the one described in
\cite[Section VI.C]{tzoumas2019outlier}). Based on the prior work of
\citet{olson2013inference}, we made the outlier covariance model isotropic with
variance $10^7$ times larger than the inlier variance and set the weights
$\omega_0,\ \omega_1$ to be the corresponding Gaussian normalizing constants. We
provide two points of comparison: a Levenberg-Marquardt (LM) solver applied to
the graph corrupted by outliers (as a ``worst case'') and the state-of-the-art
graduated nonconvexity (GNC) solver \cite{yang2019graduated}.\footnote{We use
  the GNC approach implemented in \GTSAM{} with the truncated least-squares
  cost. We use the default parameters.} Our results are summarized in Figure
\ref{fig:rpgo}. In particular, we observe that in the cases that we are able to
supply a high-quality initialization, optimization using our approach enables
recovery of accurate SLAM solutions \emph{significantly faster} than the GNC
approach (and in some cases, faster than the non-robust baseline).\footnote{The
  computation speed of our approach is primarily derived from two factors:
  first, we exploit efficient incremental optimization via iSAM2, and second,
  our optimization procedure is purely local, as opposed to GNC which requires
  solving re-weighted variants of the original pose graph optimization problem
  several times in an effort to improve robustness to initialization.} \rev{Our
  approach is susceptible to local optima (leading to suboptimal performance on
  the CSAIL dataset). We will revisit this issue in Section \ref{sec:local-opt}}.

\subsection{Tightly-coupled Semantic SLAM}\label{sec:semantic-slam}

\begin{figure}
  \centering
  \includegraphics[width=0.85\columnwidth]{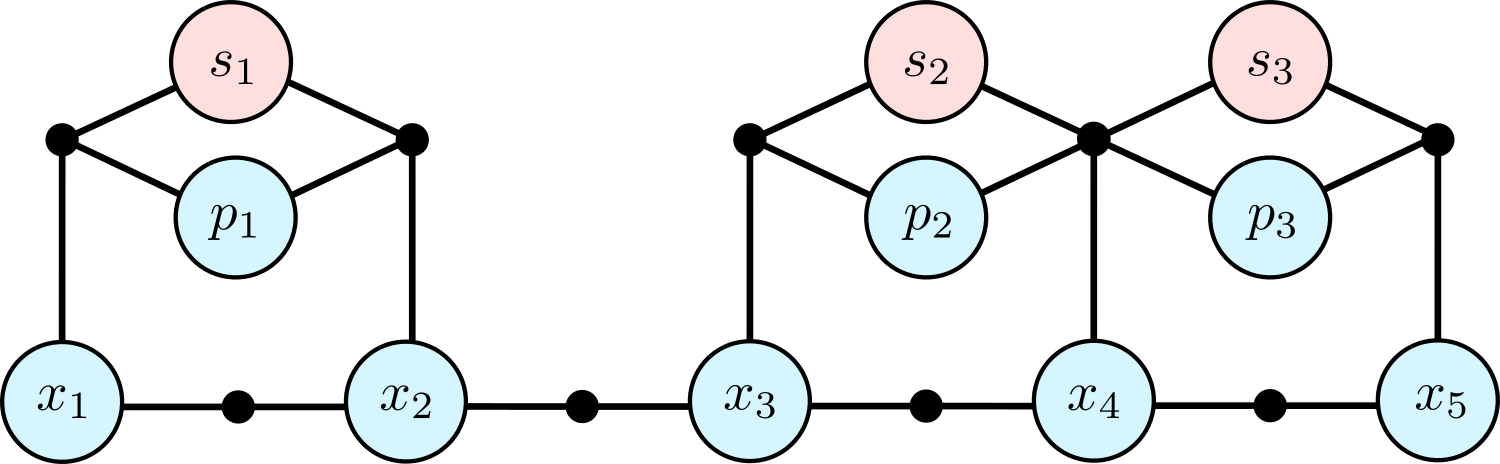}
  \caption{\rev{\textbf{Factor graph representing semantic SLAM.}} Here robot poses
    $x_i$ are connected by odometry measurements, and joint geometric-semantic
    measurements are made between poses and landmarks $(p_j, s_j)$. Measurements
    with multiple possible associations, represented as mixture factors in
    \eqref{eq:mixture-factor}, are connected to multiple landmarks. To avoid
    clutter, the measurements depicted have only two hypotheses at most, but we
    allow for larger hypothesis sets. \label{fig:semantic-slam-graph}}
\end{figure}

\begin{table}[t]
  \centering
  \noindent
  \begin{tabularx}{\columnwidth}{ c | c || c c | c c }
    \hline
    & & \multicolumn{2}{c|}{Translation Error (m)} & \multicolumn{2}{c}{Rotation Error (deg)} \\
    \hline Seq & Method & Mean  & RMSE  & Mean  & RMSE  \\
    \hline
    \multirow{2}{*}{00}
    & VISO2 \cite{geiger2011stereoscan} & 11.457 & 13.136 & 2.410 & 2.562 \\
    & Ours & 2.883 & 3.260 & 3.682 & 3.805 \\
    \hline
    \multirow{2}{*}{05}
    & VISO2 \cite{geiger2011stereoscan} & 6.227 & 7.2772 & 2.489 & 2.735 \\
    & Ours & 2.175 & 2.442 & 1.486 & 1.665 \\
    \hline
    \multirow{2}{*}{08}
    & VISO2 \cite{geiger2011stereoscan} & 8.586 & 9.797 & 3.003 & 3.283 \\
    & Ours & 8.468 & 9.429 & 6.264 & 6.732 \\
    \hline
  \end{tabularx}
  \caption{\textbf{KITTI datasets.} Absolute translation and rotation errors (in
    meters and degrees, respectively) for our approach and VISO2.}
  \label{table:kitti-stats}
  % \vspace{-5mm}
\end{table}

Several recent works have considered the problem of jointly inferring a robot's
trajectory and a set of landmark positions and classes with unknown measurement
correspondences, i.e. semantic SLAM \cite{doherty2020probabilistic,
  bowman2017probabilistic}. In particular, here we apply our approach to
optimize jointly for robot poses $x_i \in \SE(3), i=1,\ldots, n$, and landmark
locations and classes $\ell_j = (p_j, s_j),\ p_j \in \R^3,\ s_j \in
\mathcal{C},\ j = 1, \ldots, m$, where $\mathcal{C}$ is an \emph{a priori} known
set of discrete semantic classes.

We adopt the general methodology of \citet{doherty2020probabilistic} for data
association: given a range-bearing measurement and associated semantic class
obtained from an object detector, we apply a threshold on the measurement
likelihood to determine whether the measurement corresponds to a \emph{new} or
\emph{old} landmark (for this, we employ the approximate marginal computations
in \eqref{eq:discrete-marg} and \eqref{eq:laplace-marg}). If it is a new
landmark, we simply add the new measurement and landmark to our graph. If it is
an old landmark, we add it to the graph as a mixture with a single component for
each landmark that passes the likelihood threshold:
\begin{equation}\label{eq:mixture-factor}
  f_k(\pose_i, \mathcal{H}_k) \triangleq \max_{\ell_j \in \mathcal{H}_k} f_k(\pose_i, \ell_j),
\end{equation}
where $\mathcal{H}_k \subseteq L$ is a subset of landmarks and the landmark
measurement factor $f_k(\pose_i, \ell_j)$ decomposes as:
\begin{equation}\label{eq:semantic-measurement-factors}
  f_k(\pose_i, \ell_j) \triangleq \phi_k(s_j)\psi_k(\pose_i, p_j).
\end{equation}
Here $\phi_k(s_j)$ is a categorical distribution and $\psi_k(\pose_i, p_j)$ is
Gaussian with respect to the range and bearing between $\pose_i$ and $p_j$. We
also incorporate odometry factors of the form used in
\eqref{eq:pgo}.\footnote{For a more detailed exposition of this approach, see
  \cite{doherty2020probabilistic}.} The overall graphical model specifying this
problem is depicted in Figure \ref{fig:semantic-slam-graph}. Since the
measurement factors in equation \eqref{eq:semantic-measurement-factors} involve
both continuous and discrete variables, it is nontrivial to
implement this within any existing framework. In the hybrid factor graph
representation, however, problems of this form admit a concise description and
solution using discrete-continuous factors.

In this demonstration, we consider semantic SLAM using stereo camera data from
the KITTI dataset \cite{geiger2013vision}. We sample keyframes every two
seconds, using VISO2 \cite{geiger2011stereoscan} to obtain stereo odometry
measurements and YOLO \cite{redmon2016you} for noisy detections of two types of
objects: \emph{cars} and \emph{trucks}. We estimate the range and bearing to an
object's position as that of the median depth point projecting into a detected
object's bounding box. Using \DCSAM{}, we are able to compute solutions to this
problem online.\footnote{We run our solver on an Intel i7 2.6 GHz CPU and YOLO
  on an NVIDIA Quadro RTX 3000 GPU.} Table \ref{table:kitti-stats} gives a
quantitative comparison of our approach with the odometric estimate from VISO2.
Our approach substantially improves upon the translational errors of the
odometric estimate and additionally enables the estimation of discrete landmark
classes.\footnote{Qualitative results visualizing the semantic map output from
  our method are available in the supplement
  \cite{doherty2022discrete-techreport}
  .}

\begin{figure}
  \centering
  \includegraphics[width=0.9\columnwidth]{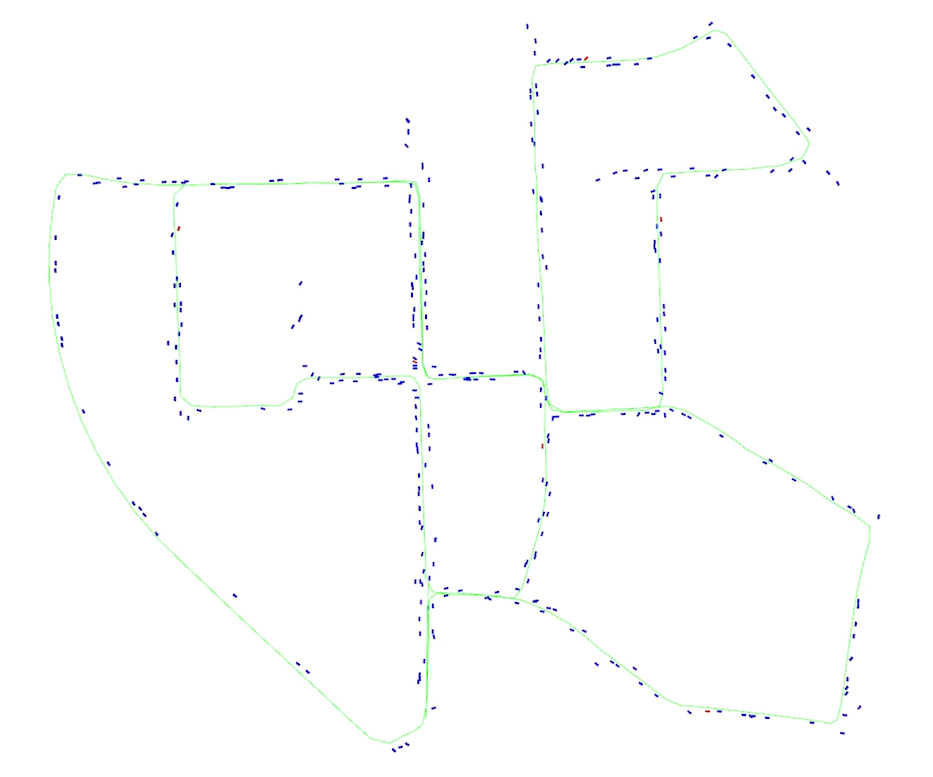}
  \caption{\textbf{KITTI Sequence 00.} Estimated trajectory (green) on the KITTI
    dataset sequence 00 using our semantic SLAM implementation. Cars are
    depicted in blue and trucks are depicted in red. \label{fig:semantic-slam-results}}
\end{figure}

\section{Discussion}\label{sec:discussion}

\subsection{When is alternating minimization efficient?}
\label{sec:efficient}

The conditional factorization in equation \eqref{eq:discrete-conditional} serves
to give some intuition for when our optimization approach is computationally
efficient. If the distribution over discrete variables conditioned on the
continuous assignment admits a factorization into small subsets
$\discreteVars_j$, then the optimization problem in \eqref{eq:discrete-opt}
decouples into separate problems in direct correspondence with each set
$\discreteVars_j$. Since we perform exact inference on this distribution,
solving for the most probable assignment is in the worst case exponential in the
size of $\discreteVars_j$ \cite{Koller09book}. Consequently, in graphs with
densely connected discrete variables that are not decoupled by continuous
variables, the per iteration complexity of alternating minimization can increase
dramatically. That said, Proposition \ref{prop:monotonic-cost-reduction} ensures
monotonic improvement in the objective so long as each optimization subproblem
admits a solution no worse than the current iterate. Therefore, it is reasonable
to consider extending this approach by allowing for \emph{local} optimization in
the discrete subproblem \cite{savchynskyy2019discrete}.

\subsection{When can we ensure accurate solutions?}
\label{sec:local-opt}

Though we are able to make some claims about when solutions to the discrete and
continuous subproblems in our alternating minimization approach can be tractably
computed, the question remains as to when one can ensure that these local search
methods recover \emph{high-quality} solutions. Since the alternating
minimization approach is a \emph{descent} method, we rely on the ability to
provide a ``good'' initial guess from which purely going ``downhill'' in the
cost landscape is enough to obtain a high-quality estimate. However, this is
already a requirement of off-the-shelf tools for solving many robot perception
problems, such as pose-graph SLAM, which (by virtue of the nonconvexity of the
optimization problems they attempt to solve) require high-quality initialization
\cite{rosen2021advances}.\footnote{\rev{Moreover, even in these ``simpler''
    problem instances, \emph{verification} that a globally optimal solution has
    been found has only been demonstrated for certain special cases (see
    \cite[Sec 2]{rosen2021advances} for a review) and is otherwise itself
    an open problem.}} Nonetheless, the consideration of discrete
variables \emph{can} make initialization more challenging. The specifics of
providing an initial guess will ultimately depend heavily on the application.

One can also attempt to reduce the initialization sensitivity of solutions
obtained by our local optimization approach. A number of methods along these
lines have been proposed. For example, graduated nonconvexity (GNC)
\cite{yang2019graduated} as discussed in Section \ref{sec:rpgo}, optimizes
nonconvex functions by successively producing (and optimizing) a more
well-behaved (typically convex) surrogate. Sampling methods and simulated
annealing methods can improve convergence by allowing for the exploration of
states that may \emph{increase} cost or by initializing a descent method like
our proposed approach from several starting points \cite{mackay1998introduction,
  szu1987fast}. Similarly, stochastic gradient descent is a classical approach
for nonconvex optimization (and has appeared in the setting of robust pose-graph
SLAM \cite{Olson06icra}), which could reasonably be adapted to our approach.
Finally, heuristics have been considered which use consistency of measurements
to filter out unlikely hypotheses \cite{mangelson2018pairwise} or to
\emph{re-initialize} estimates for factor graphs \cite{lu2021consensus}.

\section{Conclusion} 
\label{sec:conclusion}

In this work we presented an approach to optimization in discrete-continuous
graphical models based on \emph{alternating minimization}. Our key insight is
that the structure of the alternating optimization procedure allows us to
leverage the conditional independence relations exposed by factor graphs to
efficiently perform local search. We showed how the complexity of inference in
this setting is related to structure of the graphical model itself. Critically,
we observed that many important problems in robotics can be framed in terms of
graphical models admitting particularly advantageous structures for application
of our approach. We provided a method for addressing the issue of recovering
uncertainties associated with estimates in the discrete-continuous setting. Our
solver and associated tools are implemented as part of our library, \DCSAM{},
which is, to the best of our knowledge, the first openly available library for
addressing these \emph{hybrid} discrete-continuous optimization problems.
Finally, we demonstrate the application of our method to the key problems of
robust pose graph optimization, and semantic SLAM.

\bibliographystyle{IEEEtranN}
\begin{footnotesize}
\bibliography{references}
\end{footnotesize}

\end{document}